\documentclass{article}

\PassOptionsToPackage{numbers, compress}{natbib}
\usepackage[preprint]{neurips_2025}

\usepackage[utf8]{inputenc} 
\usepackage{graphicx} 
\usepackage[T1]{fontenc}    
\usepackage{hyperref}       
\usepackage{url}            
\usepackage{booktabs}       
\usepackage{amsfonts}       
\usepackage{nicefrac}       
\usepackage{microtype}      
\usepackage{xcolor}         

\usepackage{amsmath}
\usepackage{amssymb}
\usepackage{mathtools}
\usepackage{amsthm}
\usepackage[capitalize,noabbrev]{cleveref}
\usepackage{subcaption}
\usepackage{amsmath}
\usepackage{amssymb}
\usepackage{mathtools}
\usepackage{amsthm}
\usepackage[capitalize,noabbrev]{cleveref}
\theoremstyle{plain}
\newtheorem{theorem}{Theorem}[section]

\newtheorem{lemma}[theorem]{Lemma}

\theoremstyle{definition}

\theoremstyle{remark}
\newtheorem{remark}[theorem]{Remark}

\usepackage{natbib}
\usepackage[textsize=tiny]{todonotes}
\usepackage{hyperref}
\usepackage{url}
\usepackage{float}
\usepackage{graphicx} 
\usepackage{algorithmic}

\usepackage{amssymb}
\usepackage{amsmath}
\usepackage{amsthm}
\usepackage{enumitem}
\usepackage{caption}
\usepackage{multirow}
\usepackage{graphicx}
\usepackage[dvipsnames]{xcolor}

\def\int{\displaystyle\mathop {\mbox{\rm int}}}    

\usepackage{makecell}

\DeclareUnicodeCharacter{2212}{\ensuremath{-}}

\usepackage{algorithm,algorithmic}

\newcommand{\bx}{{\bf x}}

\usepackage{comment}

\usepackage{amsmath,bm} 

\makeatletter
\newcommand*{\rom}[1]{\expandafter\@slowromancap\romannumeral #1@}
\makeatother

\usepackage{wrapfig}
\usepackage{multicol}

\title{On-the-Fly OVD Adaptation with FLAME: \underline{F}ew-shot \underline{L}ocalization via \underline{A}ctive \underline{M}arginal-Samples \underline{E}xploration}

\author{%
  \textbf{Yehonathan Refael},
  \textbf{Amit Aides},
  \textbf{Aviad Barzilai},\\
  \textbf{George Leifman, Vered Silverman, Bolous Jaber, Tomer Shekel,
  Genady Beryozkin}\\
  Google Research
}

\begin{document}

\maketitle

\begin{abstract}
Open-vocabulary object detection (OVD) models offer remarkable flexibility by detecting objects from arbitrary text queries. However, their zero-shot performance in specialized domains like Remote Sensing is often compromised by the inherent ambiguity of natural language, limiting critical downstream applications. For instance, an OVD model may struggle to distinguish between fine-grained classes such as "fishing boat" and "yacht" since their embeddings are similar and often hard to separate. This can hamper specific user goals, such as monitoring illegal fishing, by producing irrelevant detections. To address this, we propose a cascaded approach that couples the broad generalization of a large pre-trained OVD model with a lightweight few-shot classifier. Our method first employs the zero-shot model to generate high-recall object proposals. These proposals are then refined for high precision by a compact classifier trained in real-time on only a handful of user-annotated examples - drastically reducing the high costs of remote sensing imagery annotation. The core of our framework is FLAME, a one-step active learning strategy that selects the most informative samples for training. FLAME identifies, on the fly, uncertain marginal candidates near the decision boundary using density estimation, followed by clustering to ensure sample diversity. This efficient sampling technique achieves high accuracy without costly full-model fine-tuning and enables instant adaptation, within less than a minute, which is significantly faster than state-of-the-art alternatives. Our method consistently surpasses state-of-the-art performance on remote sensing benchmarks, establishing a practical and resource-efficient framework for adapting foundation models to specific user needs.
\end{abstract}
\begin{figure*}[hbt]
    \centering
    \includegraphics[width=0.75\linewidth]{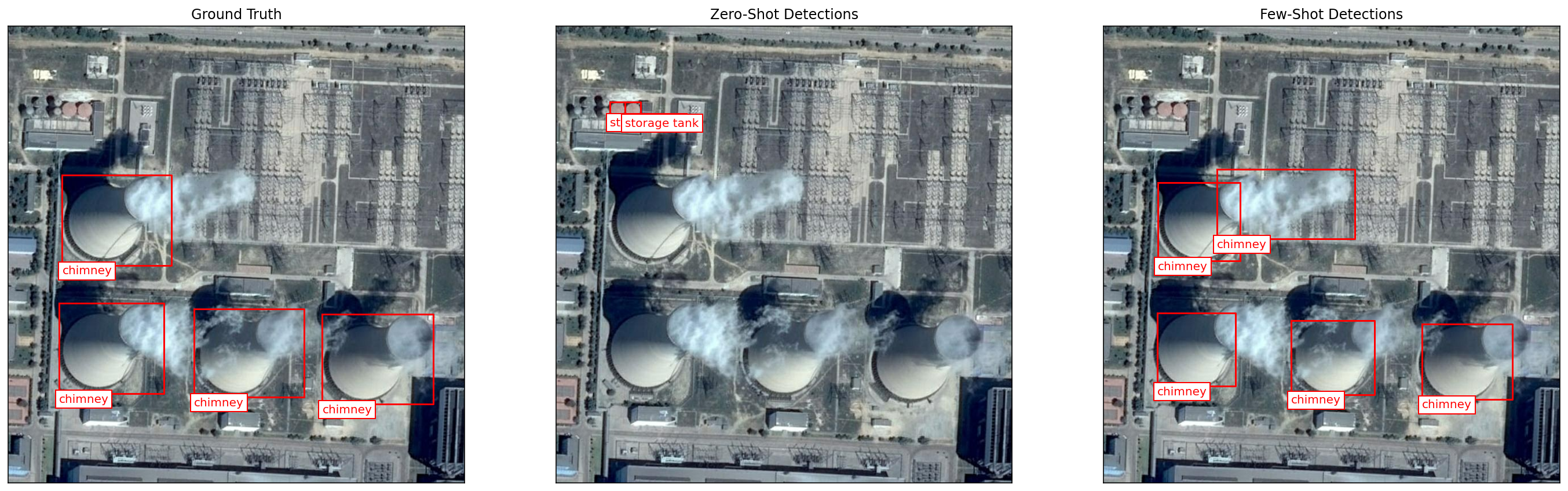}\\
    \includegraphics[width=0.75\textwidth]{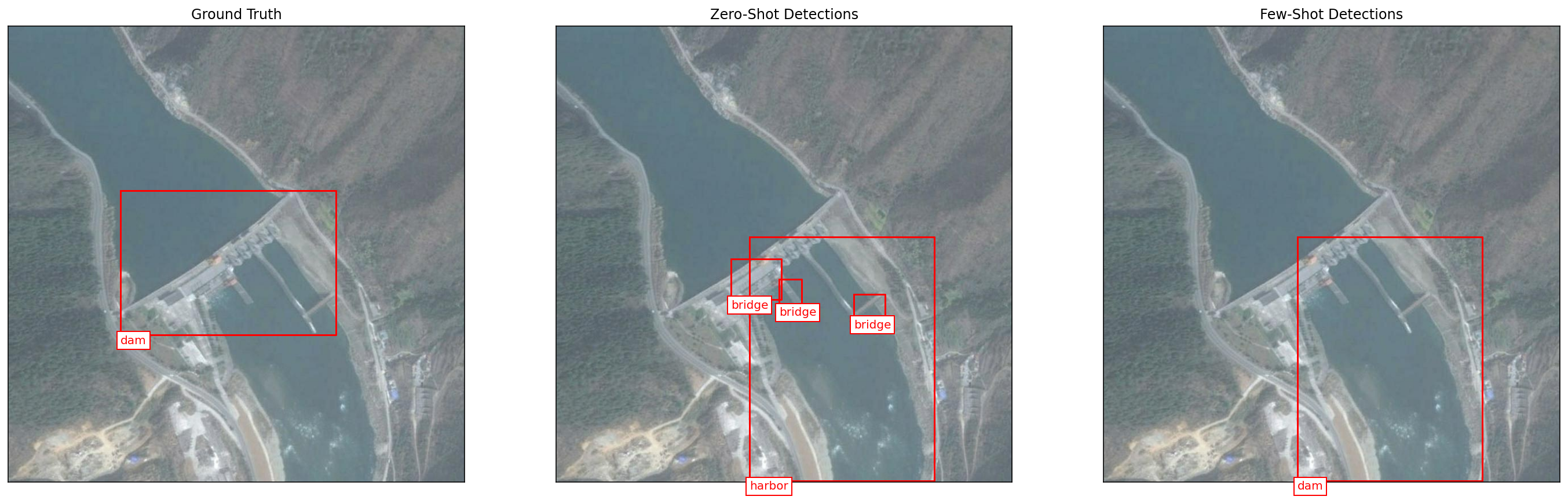}
    \caption{A visual demonstration of performance improvement from Zero-Shot to Few-Shot detection using DIOR dataset \cite{dior2023rs}. The Zero-Shot model (center) produces noisy and unreliable results, identifying the 'chimneys' but with low confidence and accompanied by several false positives. Our Few-Shot method (right) refines this output, successfully eliminating the false positives and accurately detecting all four chimneys shown in the Ground Truth (left). Similarly, the bottom row showcases the detection of a 'dam'. The Zero-Shot model struggles with false positives such as 'bridge' and 'harbor', which are corrected by the more precise Few-Shot approach.}\label{fig::chimney_detections}
\end{figure*}
\section{Introduction}
The recent advancements in large-scale vision-language models (VLMs) such as CLIP~\cite{radford2021learning} have catalyzed a paradigm shift in computer vision, giving rise to Open-Vocabulary Object Detection (OVD)~\cite{zareian2021open}. Unlike traditional detectors limited to predefined categories, OVD models can identify objects described by arbitrary natural language text, offering unprecedented flexibility. This is particularly transformative for remote sensing, where cataloging every possible class is intractable. Early OVD methods adapted standard detectors by replacing the classifier head with text embeddings~\cite{gu2021open}, leveraging the semantic richness of VLMs to generalize to unseen categories. However, the inherent ambiguity of text queries often leads to significant drops in precision, limiting the utility of pure zero-shot systems. 

To overcome the limitations of pure zero-shot systems, one alternative is Few-Shot Object Detection (FSOD) \cite{kang2019few}, which adapts models to novel categories using only a handful of annotated examples. In remote sensing,  FSOD is critical due to the difficulty and cost of acquiring dense labels\cite{barzilai2025recipeimprovingremotesensing}. While effective, common FSOD strategies like meta-learning or fine-tuning \cite{wang2020frustratingly} can be computationally intensive. To address this, Parameter-Efficient Fine-Tuning (PEFT) techniques such as LoRA \cite{hu2021lora} have emerged to alleviate these costs by reducing the number of trainable parameters.

These FSOD and PEFT strategies are primarily designed to create specialized detectors optimized for a new, specific set of target classes, for example \cite{bou2024exploring, jeune2023, le2022improving} are tailored for RS. However, these more efficient adaptation methods still involve a computationally demanding fine-tuning step. Even some recent prototype-based methods \cite{bou2024exploring} require tuning for hundreds of epochs, a process that can take hours and necessitates an accelerator like a GPU (a phase our proposed method eliminates, as we demonstrate later in this study). 

A distinct paradigm explores a hybrid approach that merges OVD and FSOD, using few-shot supervision to enhance and expand an open-vocabulary detector’s existing knowledge within a single, unified framework \cite{cheng2024revisiting}. Several strategies explore this hybrid model: prompt-based methods \cite{feng2022promptdet,zhang2022tip} learn continuous prompts from support sets to improve category alignment, while Transformer-based method like OV-DETR \cite{zang2022ovdetr} shows strong generalization.

The success of these hybrid approaches, which use only a handful of examples, hinges on the efficient selection of the most informative ones.
This challenge is addressed by Active Learning (AL)~\cite{settles2009active}, which queries an oracle for the most beneficial labels. Common AL strategies include uncertainty-based sampling~\cite{lewis1994sequential}, diversity-based sampling~\cite{sener2018active}, or their combinations~\cite{choi2021active}. Building on this foundation, our work proposes a cascaded OVD–FSOD framework with a novel AL strategy specifically designed to resolve semantic ambiguity in remote sensing imagery efficiently and effectively.
\begin{figure*}[htb]
    \centering
    \includegraphics[width=0.95\linewidth]{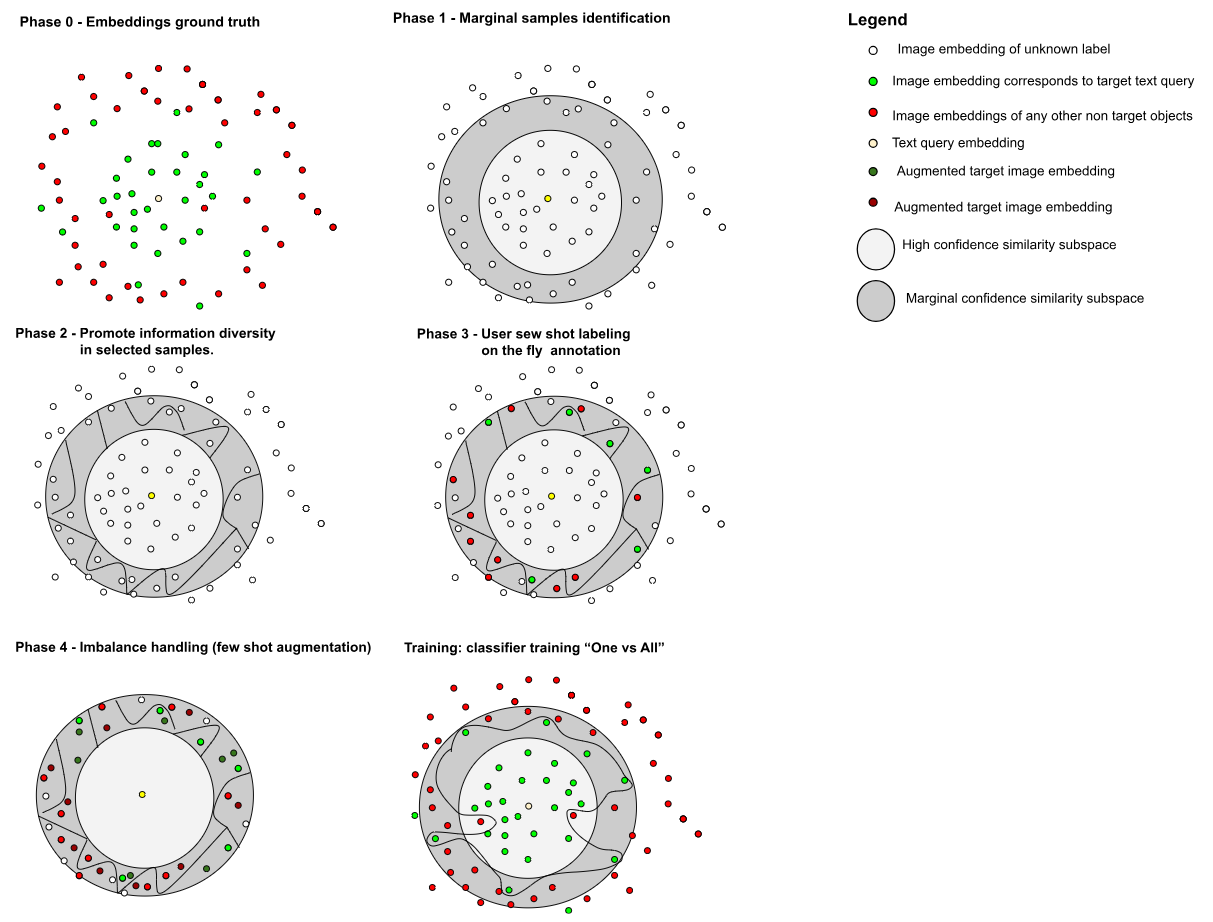}
    \caption{Overview of the proposed few-shot sampling method. The method follows the stages: (1) uncertainty-based filtering using density estimation to identify ambiguous candidates near the decision boundary, (2) clustering-based diversity sampling to ensure representative coverage, (3) interactive user annotation of the selected samples, (4) conditional data augmentation with SMOTE or SVM-SMOTE to balance classes, and (5) lightweight classifier training (e.g., SVM or MLP) on the augmented set. This cascaded process refines the zero-shot proposals from a large open-vocabulary detector into an accurate, real-time few-shot classifier without full-model fine-tuning}
    \label{fig::Method_illustration}
\end{figure*}
\section{Method}
\begin{algorithm*}[hbtp!]
\caption{FLAME: \underline{F}ew-shot \underline{L}ocalization via \underline{A}ctive \underline{M}arginal-Samples \underline{E}xploration}
\label{alg:fewshot-ovd}
\begin{algorithmic}[1]
\REQUIRE Unlabeled pool of embeddings $X=\{x_i\}_{i=1}^N\subset\mathbb{R}^d$, text embedding $t\in\mathbb{R}^d$; number of target shots $K$; PCA dimension $\ell$; Hyperparameters: Gaussian KDE bandwidth $h$, ratios $0<r_l<r_u<1$, imbalance threshold $\tau$.
\ENSURE Selected shots $\hat{X}:=\{\hat{x}_k\}_{k=1}^K$ 

\FOR{$i=1$ \TO $N$}
  \STATE Compute cosine similarities: $c_i \leftarrow \dfrac{x_i^\top t}{\|x_i\|\,\|t\|}$
  \STATE Augment examples: $\tilde{x}_i \leftarrow [x_i,\,c_i]$
\ENDFOR\\
{\color{blue} \# Marginal samples identification}
\STATE Project $\{\tilde{x}_i\}$ to $\ell$ dimensions via PCA to get $S=\{s_i\}_{i=1}^N$ 
\STATE Fit Gaussian KDE $\hat{f}$ (bandwidth $h$) on $S$:  $s^\star \leftarrow \arg\max_s \hat{f}(s)$
\STATE Find samples density boundaries $s_L,s_U$ s.t. 
$\hat{f}(s_L)=r_l\hat{f}(s^\star)$, and  $\hat{f}(s_U)=r_u\hat{f}(s^\star)$\\
{\color{blue} \# Promote information diversity}
\STATE Set $\mathcal{I}_{\text{marginal}} \leftarrow \{i \mid s_i\in[s_L,s_U]\}$, \quad $X_{\text{marginal}} \leftarrow \{x_i \mid i\in\mathcal{I}_{\text{marginal}}\}$
\STATE Run $k$-means clustering on $X_{\text{marginal}}$ into $K$ clusters $\{C_k\}_{k=1}^K$ 
\STATE Find examples closest to each center $\hat{X} \leftarrow \{\hat{x}_k\}_{k=1}^K$
\STATE 
{\color{blue} \# User few shot labeling}
\STATE User labels the few-shots $\hat{X}$ to obtain $D_{\text{labeled}}=\{(\hat{x}_k,y_k)\}_{k=1}^K$, $y_k\in\{0,1\}$\\
{\color{blue} \# Imbalance handling}
\STATE Compute imbalance ratio 
       $\rho \leftarrow \dfrac{\max_{c\in\{0,1\}} |\{y_k=c\}|}{\min_{c\in\{0,1\}} |\{y_k=c\}|}$
\IF{$\rho > \tau$}
  \STATE $\hat{X} \leftarrow \text{SMOTE}(D_{\text{labeled}})$
\ENDIF

\STATE \textbf{return} $\ \hat{X}$
\end{algorithmic}
\end{algorithm*}
\paragraph{Theory and Motivations.} Our framework lays on the observation that a binary classifier, whether an SVM \citep{vapnik1995support} or a positively homogeneous neural network \citep{polyakov2023homogeneousartificialneuralnetwork}, can be determined entirely by its margin (support) examples. Equivalently, if one removes all non-support training points and retrains, the resulting classifier is unchanged. Building on this, our few-shot procedure identifies a small set of near-boundary examples (the "few-shots"), asks the user to label them, and trains a lightweight model on the fly. Despite using only a handful of points, this model matches the classifier that would have been obtained from training on the full dataset, which may be too large or impractical for real-time training. The lemmas below formalize this fact for the SVM, soft margin SVM and for neural networks.  The proofs of the Lemmas are relegated to the Appendix~\ref{sec::proofs}.

\begin{lemma}[Support--determination for hard–margin SVM]\label{lem::margin_svm}
Let $\{(x_i,y_i)\}_{i=1}^n$ be linearly separable with $y_i\in\{\pm1\}$.
Consider the hard–margin SVM
\begin{equation}
\label{eq:svmP}
\min_{w,b}\ \tfrac12\|w\|^2
\;\text{s.t.}\; y_i\,(w^\top x_i + b)\ \ge\ 1,\;(i=1,\dots,n).
\tag{P}
\end{equation}
Let $(w^\star,b^\star)$ be an optimal solution and define the support set
$S \;:=\; \bigl\{ i\in[n] : y_i\,(w^{\star\top}x_i+b^\star)=1 \bigr\}.$
Then,
\begin{enumerate}
\item $(w^\star,b^\star)$ together with multipliers $\{\alpha_i^\star\}_{i\in S}$
forms a Karush-kuhn-tucker (KKT) \citep{ciano2024karush} pair for the \emph{reduced} problem that retains only constraints
indexed by $S$:
\begin{equation}
\label{eq:svmPS}
\min_{w,b}\ \tfrac12\|w\|^2
\quad\text{s.t.}\; y_i\,(w^\top x_i + b)\ \ge\ 1,\;(i\in S).
\tag{$\mathrm{P}_S$}
\end{equation}
\item Conversely, if $(\tilde w,\tilde b)$ and multipliers $\{\mu_i\}_{i\in S}$ satisfy
the KKT system of \eqref{eq:svmPS}, then extending the multipliers by
$\tilde\alpha_i:=\mu_i$ for $i\in S$ and $\tilde\alpha_i:=0$ for $i\notin S$
yields a KKT pair $(\tilde w,\tilde b,\tilde\alpha)$ for the full problem
\eqref{eq:svmP}.
\end{enumerate}
Consequently, \eqref{eq:svmP} and \eqref{eq:svmPS} have the same optimal solutions.
In particular, retraining the hard–margin SVM after removing all non–support points
$[n]\setminus S$ leaves the classifier $x\mapsto\operatorname{sign}(w^\top x+b)$ unchanged.
\end{lemma}

\begin{remark}[Kernel SVM]
The same argument holds verbatim for kernel SVMs by replacing $x_i$ with
$\varphi(x_i)$ in a feature space: at optimality $w^\star=\sum_{i\in S}\alpha_i^\star y_i\varphi(x_i)$,
so only support vectors ($\alpha_i^\star>0$) determine the classifier.
\end{remark}

Our claim for the non-separable embeddings case, which is the soft marginal SVM, is stated in the following lemma.
\setcounter{theorem}{1}
\begin{lemma}[Support-determination for soft-margin SVM]\label{lem::soft_margin_svm}
Let $\{(x_i,y_i)\}_{i=1}^n$ be possibly non-separable with $y_i\in\{\pm1\}.$ Consider a penalty parameter $C > 0,$ then the soft-margin SVM is formulated by
\begin{align*}\label{eq:svmP_}
\min_{w,b,\xi} &\quad \frac{1}{2}\|w\|^2 + C \sum_{i=1}^n \xi_i \\
&\text{s.t.} \quad y_i(w^\top x_i + b) \geq 1 - \xi_i, \quad (i=1, \dots, n) 
\tag{$P$}
\end{align*}
Let $(w^*, b^*, \xi^*)$ be an optimal solution to the soft-margin problem (\ref{eq:svmP}) with corresponding dual multipliers $\{\alpha_i^*\}_{i=1}^n$ and $\{\beta_i^*\}_{i=1}^n$. Define the support set $S$ as the set of indices with non-zero multipliers $\alpha_i^*$, $S := \{ i \in [n] \mid \alpha_i^* > 0 \}$
. Then,
\begin{enumerate}
    \item $(w^*, b^*, \{\xi_i^*\}_{i \in S})$ together with multipliers $\{\alpha_i^*, \beta_i^*\}_{i \in S}$ forms a Karush-kuhn-tucker (KKT) \citep{ciano2024karush}  pair for the reduced problem that retains only constraints indexed by $S$:
    \begin{align*}\label{eq:svmPS_}
    &\min_{w,b,\xi} 
    \quad \frac{1}{2}\|w\|^2 + C \sum_{i=1}^n \xi_i\\ &\quad 
    \text{s.t.}
    \quad y_i(w^\top x_i + b) \geq 1 - \xi_i, \quad (i \in S) 
    \tag{$P_S$} 
    \end{align*}
    \item Conversely, if $(\tilde{w}, \tilde{b}, \{\tilde{\xi}_i\}_{i \in S})$ and multipliers $\{\tilde{\alpha}_i, \tilde{\beta}_i\}_{i \in S}$ satisfy the KKT system for (\ref{eq:svmPS_}), then extending the solution by setting $\tilde{\alpha}_i = 0$, $\tilde{\xi}_i = 0$, and $\tilde{\beta}_i = C$ for all $i \notin S$ yields a full KKT pair $(\tilde{w}, \tilde{b}, \tilde{\xi}, \tilde{\alpha}, \tilde{\beta})$ for the full problem (\ref{eq:svmP_}).
\end{enumerate}
Consequently, (\ref{eq:svmP_}) and (\ref{eq:svmPS_}) have the same optimal solutions $(w, b)$. Retraining the soft-margin SVM after removing all non-support points ($i \notin S$) leaves the classifier $x \mapsto \operatorname{sign}(w^\top x + b)$ unchanged.
\end{lemma}

Lastly, the following Lemma formalizes our claim for the case of neural network. 
\begin{lemma}[Support examples--determination for homogeneous networks]\label{lem::margin_homogeneous_networks}
Let $\Phi(\theta;\cdot)$ be binary classifier $L$-homogeneous\footnote{A network $\Phi(\theta;x)$ is called \emph{homogeneous} of degree $c>0$ if
for all $b>0$ and all $\theta,x$, it holds that $\Phi(b,\theta;\bx)=b^{c}\,\Phi(\theta;x)$.}
 in the weights parameters $\theta$
(e.g., ReLU, Leaky ReLU, sigmoid etc), and let the binary training set
$\{(x_i,y_i)\}_{i=1}^n$ be linearly separable by $\Phi(\theta;\cdot)$.
Consider gradient flow on logistic loss and assume it converges in \emph{direction}
to a KKT point $(\theta^\star,\lambda^\star)$ of the maximum–margin program
\begin{equation}
\min_{\theta}\ \tfrac12\|\theta\|^2
\quad\text{s.t.}\quad y_i\,\Phi(\theta;x_i)\ \ge 1\quad (i=1,\dots,n).
\label{eq:mm}
\end{equation}
Let the (margin/support) set be
$S \;:=\; \{\, i\in[n] : y_i\,\Phi(\theta^\star;x_i)=1 \,\}.$
Then,
\begin{enumerate}
\item $(\theta^\star,\{\lambda_i^\star\}_{i\in S})$ satisfies the KKT system of the
\emph{reduced} problem that keeps only constraints with indices in $S$:
\begin{equation}
\min_{\theta}\ \tfrac12\|\theta\|^2
\quad\text{s.t.}\quad y_i\,\Phi(\theta;x_i)\ \ge 1\quad (i\in S).
\label{eq:mmS}
\end{equation}
\item Conversely, if $(\tilde\theta,\{\mu_i\}_{i\in S})$ is a KKT pair for
\eqref{eq:mmS} and we define $\tilde\lambda_i:=\mu_i$ for $i\in S$ and
$\tilde\lambda_i:=0$ for $i\notin S$, then $(\tilde\theta,\tilde\lambda)$ is a KKT
pair for the full problem \eqref{eq:mm}.
\end{enumerate}
Consequently, the sets of KKT solutions of \eqref{eq:mm} and \eqref{eq:mmS} coincide.
In particular, retraining after removing all non-support points $[n]\setminus S$
produces the same limiting classifier $x\mapsto\operatorname{sign}(\Phi(\theta;x))$.
\end{lemma}

\begin{table*}[htb]
\centering
\caption{Comparison of few-shot object detection performance on the DOTA and DIOR datasets, based on 30-shot examples. The metric used is Average Precision (AP). Our proposed method achieves state-of-the-art results while demonstrating a significantly faster adaptation time.}
\begin{tabular}{l|cc|cc}
\toprule
\multirow{2}{*}{Method} & \multicolumn{2}{c|}{DOTA} & \multicolumn{2}{c}{DIOR} \\
& & & & \\
\midrule
Zero-shot OWL-ViT-v2 (Baseline) & 13.774\% & & 14.982\% & \\
\midrule
Zero-shot RS-OWL-ViT-v2 & 31.827\% & & 29.387\% & \\
\midrule
\midrule
Jeune et. al \cite{le2022improving} & 37.1\% & & 35.6\% & \\
\midrule
SIoU \cite{jeune2023} & 45.88\% & & 52.85\% & \\
\midrule
Prototype-based FSOD with DINOv2 \cite{bou2024exploring} & 41.40\% & & 26.46\% & \\
\midrule
\textbf{FLAME cascaded on RS-OWL-ViT-v2} & \textbf{53.96\%} & & \textbf{53.21\%} & \\
\bottomrule
\end{tabular}\label{tab:results_comparison}
\end{table*}
\paragraph{Marginal Samples Retrieval.} We propose a one-stage active learning strategy that pinpoints the most informative samples for training a lightweight, class-specific binary classifier. Algorithm~\ref{alg:fewshot-ovd} allows a large-scale, zero-shot OVD model to be adapted to a new target class efficiently, in real-time, and with minimal human supervision. The method is illustrated in Figure~\ref{fig::Method_illustration}.
First, we identify uncertain candidates by augmenting image embeddings with their zero-shot similarity to the text query and applying density estimation in a projected (PCA) augmented-embedding-space. Samples at the distribution’s margins are retained as they carry the most informative ambiguity. From this pool, we promote diversity by clustering and selecting one representative per cluster, yielding $K$ candidate shots for annotation. The user then labels these few informative samples, forming an initial dataset. To mitigate imbalance, we apply Synthetic Minority Over-sampling Technique (SMOTE) \cite{chawla2002smote} for extremely fast augmentation. This procedure would contribute to a balanced, representative, and efficient training to take place shortly after.
\begin{table*}[htb!]
\centering

\setlength{\tabcolsep}{4pt} 

\caption{Detailed per-class Average Precision (AP) comparison of our few-shot method against a zero-shot baseline (OWL-ViT-v2 fine-tuned on RS-WebLI) on the DIOR (left) and DOTA (right) datasets. The '–' symbol denotes a failure case for our method, occurring when the initial zero-shot step retrieved no relevant candidate images for a given class, thereby preventing the few-shot selection process. The results highlight the substantial AP gains achieved by our approach across a diverse range of object categories.}

\label{tab:dior_dota_side_by_side_revised}

\begin{minipage}[hb]{0.5\textwidth}
\centering
\textbf{DIOR Dataset}
\begin{tabular}{p{3.5cm}cc}
\toprule
\textbf{Class} & \textbf{Zero Shot} & \textbf{Few Shot} \\
\midrule
expressway service area & 0.03 & 0.82 \\
expressway toll station & 0 & 0.99 \\
airplane & 0.84 & 0.99 \\
airport & 0 & -- \\
baseball field & 0.62 & 0.93 \\
basketball court & 0.66 & 0.87 \\
bridge & 0.21 & 0.49 \\
chimney & 0.11 & 0.94 \\
dam & 0.04 & 0.71 \\
golf field & 0.01 & 0.72 \\
ground track field & 0.5 & 0.79 \\
harbor & 0.33 & 0.64 \\
overpass & 0.1 & 0.75 \\
ship & 0.72 & 0.93 \\
stadium & 0.57 & 0.86 \\
storage tank & 0.73 & 0.68 \\
tennis court & 0.8 & 0.57 \\
train station & 0.01 & -- \\
vehicle & 0.25 & 0.79 \\
windmill & 0.67 & 1 \\
\bottomrule
\end{tabular}
\end{minipage}
\hfill 
\begin{minipage}[htb]{0.45\textwidth}
\centering
\textbf{DOTA Dataset}
\begin{tabular}{p{3cm}cc}
\toprule
\textbf{Class} & \textbf{Zero Shot} & \textbf{Few Shot} \\
\midrule
Baseball Diamond & 0.32 & 0.88 \\
Basketball Court & 0.56 & 0.83 \\
Bridge & 0.09 & 0.28 \\
Container Crane & 0.03 & 0.95 \\
Ground Track Field & 0.4 & 0.68 \\
Harbor & 0.36 & 0.82 \\
Helicopter & 0.39 & 0.73 \\
Large Vehicle & 0.32 & 0.87 \\
Plane & 0.78 & 0.54 \\
Roundabout & 0.24 & 0.91 \\
Ship & 0.71 & 0.82 \\
Small Vehicle & 0.28 & 0.77 \\
Soccer Ball Field & 0.48 & 0.77 \\
Storage Tank & 0.79 & 0.55 \\
Swimming Pool & 0.71 & 0.58 \\
Tennis Court & 0.77 & 0.01 \\
\bottomrule
\end{tabular}
\end{minipage}
\end{table*}

Finally, using the (augmented) few-shots returned by Algorithm \ref{alg:fewshot-ovd}, we train a compact classifier, by default a Radial Basis kernel (RBF) SVM \citep{scholkopf1999advances}, which is trained to find a non-linear separating hyperplane. Note that our efficient framework could support many lightweight alternatives such as: Two-Layer Multi-Layer Perceptron (MLP) under binary cross-entropy loss function, or encoder-classifier with Triplet Loss \citep{dong2018triplet}. Illustration schema of the algorithm is presented in Figure~\ref{fig::Method_illustration}.

\section{Experiments}

To evaluate its performance, our few-shot method is benchmarked against a zero-shot baseline and leading state-of-the-art approaches, as summarized in Table \ref{tab:results_comparison}.
To that end, we leverage the following two remote sensing datasets: (1) DOTA \cite{xia2018dota} (Dataset for Object Detection in Aerial Images): A large-scale remote sensing dataset with multi-class, multi-oriented objects annotated in high-resolution aerial images for object detection; (2) DIOR \cite{li2020object} (Dataset for Object Detection in Optical remote sensing Images): A diverse large-scale dataset of optical remote sensing images containing numerous object categories across varying conditions and resolutions for robust detection. 

We first evaluate the zero-shot performance of the baseline OWL-ViT-v2 model \cite{minderer2024}, which was pre-trained on the vast, generic multilingual WebLI dataset \cite{chen2023palijointlyscaledmultilinguallanguageimage}. We then consider the RS-OWL-ViT-v2 model, a remote sensing variant of OWL-ViT-v2 fine-tuned on the RS-WebLI dataset \cite{barzilai2025recipeimprovingremotesensing}, which consists of three million aerial and satellite images from the original WebLI dataset and on a collection of $67,000$ aerial images annotated for remote sensing object detection across $34$ categories. This improved zero-shot performance model serves as the starting point for FLAME.

Table \ref{tab:results_comparison} demonstrates that the FLAME cascaded on RS-OWL-VIT-v2 method achieves the highest Average Precision (AP) on both the DOTA ($53.96\%$) and DIOR ($53.21\%$) datasets among all compared Few-Shot Object Detection (FSOD) models.This superior performance is coupled with a significantly faster adaptation time (approximately 1 minute per label on a CPU) compared to competing fine-tuning approaches that typically require a hardware accelerator (such as TPU or GPU) and several hours.

Following, Table \ref{tab:dior_dota_side_by_side_revised} provides a detailed per-class breakdown of the Average Precision (AP) on both the DIOR and DOTA datasets, comparing our few-shot method against the zero-shot baseline using the Zero-shot RS-OWL-VIT-v2 fine-tuned on RS-WebLI (which appear in second line of Table \ref{tab:results_comparison}). 
The missing values in the 'Few-Shot' columns indicate instances where the initial zero-shot retrieval step failed to find any relevant image embeddings. Without these initial candidates, the few-shot selection process could not proceed, resulting in a method failure for those specific classes. The Table demonstrates the substantial performance gains achieved by the Few-Shot (FLAME) method over the Zero-Shot baseline across a wide range of object categories on both the DIOR and DOTA datasets. For instance, the Few-Shot method dramatically improves AP for challenging classes like 'expressway toll station' on DIOR (from $0\%$ to $99\%$) and 'Container Crane' on DOTA (from $3\%$ to $95\%$), showcasing its effectiveness in resolving semantic ambiguity

\section{Discussion}

Remote sensing is a field that involves the acquisition of information about an object or area without making physical contact with it, typically using sensors on platforms such as satellites or aircraft. The proposed method provides a practical and resource-efficient framework for adapting foundational remote sensing OVD models to specific user needs. The cascaded architecture combines a large, pre-trained OVD model with a lightweight, few-shot classifier. This approach generates initial object-embedding proposals using the frozen weights of the zero-shot model, which are then refined by a compact classifier trained in real-time on a handful of user-annotated examples. This process drastically reduces annotation overhead while achieving high accuracy without the costly process of full-model fine-tuning. The core contribution is an efficient one-step active learning strategy that selects the most informative samples for user annotation. This strategy identifies a small number of uncertain candidates near the decision boundary using density estimation and then applies clustering to ensure a diverse training set. The method is designed to address the semantic ambiguity of text queries that hampers the zero-shot performance of pre-trained models.


\bibliographystyle{plain} 
\bibliography{references} 

\appendix
\section{Proofs}\label{sec::proofs}
\begin{proof}[Proof of Lemma \ref{lem::margin_svm}]
Introduce multipliers $\alpha_i\ge 0$ for the constraints in \eqref{eq:svmP}.  
The Lagrangian is
\[
\mathcal{L}(w,b,\alpha)=\tfrac12\|w\|^2-\sum_{i=1}^n \alpha_i\bigl(y_i(w^\top x_i+b)-1\bigr),
\]
and the KKT conditions are
\begin{align*}
&\text{(stationarity)} && w=\sum_{i=1}^n \alpha_i y_i x_i,\qquad \sum_{i=1}^n \alpha_i y_i=0,\\
&\text{(primal feas.)} && y_i(w^\top x_i+b)\ge 1\quad(\forall i),\\
&\text{(dual feas.)} && \alpha_i\ge 0\quad(\forall i),\\
&\text{(comp.\ slackness)} && \alpha_i\bigl(y_i(w^\top x_i+b)-1\bigr)=0\quad(\forall i).
\end{align*}

\emph{(1) Full $\Rightarrow$ reduced.}
Let $(w^\star,b^\star,\alpha^\star)$ be any KKT triple for \eqref{eq:svmP}, and set
$S=\{i:\,y_i(w^{\star\top}x_i+b^\star)=1\}$. By complementary slackness,
$\alpha_i^\star=0$ for every $i\notin S$. Hence stationarity reduces to
\[
w^\star=\sum_{i\in S}\alpha_i^\star y_i x_i,\qquad \sum_{i\in S}\alpha_i^\star y_i=0,
\]
and together with feasibility and slackness on $S$ these are exactly the KKT
conditions of the reduced problem \eqref{eq:svmPS}. Thus
$\bigl(w^\star,b^\star,(\alpha_i^\star)_{i\in S}\bigr)$ is KKT for \eqref{eq:svmPS}.

\emph{(2) Reduced $\Rightarrow$ full.}
Conversely, let $(\tilde w,\tilde b,(\mu_i)_{i\in S})$ satisfy the KKT system for
\eqref{eq:svmPS}, and define $\tilde\alpha_i:=\mu_i$ for $i\in S$ and
$\tilde\alpha_i:=0$ for $i\notin S$. Then stationarity, dual feasibility, and
complementary slackness for \eqref{eq:svmP} hold immediately. To check the remaining
primal feasibility on $[n]\setminus S$, compare duals: the dual of \eqref{eq:svmPS}
is the dual of \eqref{eq:svmP} restricted to indices $S$. Since an optimal dual
solution of \eqref{eq:svmP} has $\alpha_i^\star=0$ for $i\notin S$, the restricted
dual attains the same optimal value; by strong duality, \eqref{eq:svmP} and
\eqref{eq:svmPS} share the same optimal objective value. Because the primal objective
is strictly convex in $w$, any optimal reduced solution must satisfy $\tilde w=w^\star$,
and the equalities on $S$ then fix $\tilde b=b^\star$. Hence
$y_i(\tilde w^\top x_i+\tilde b)\ge 1$ for all $i\in[n]$, i.e., primal feasibility
for the full problem. Thus $(\tilde w,\tilde b,\tilde\alpha)$ is KKT for \eqref{eq:svmP}.

Parts (1)–(2) imply that \eqref{eq:svmP} and \eqref{eq:svmPS} have the same optimal
solutions. In particular, removing non–support points leaves the classifier
$x\mapsto\operatorname{sign}(w^\top x+b)$ unchanged.
\end{proof}

\begin{proof}[Proof of Lemma \ref{lem::soft_margin_svm}]
Let $(w^*,b^*,\xi^*;\alpha^*,\beta^*)$ be a KKT pair of \eqref{eq:svmP}, where the Lagrangian is
$\mathcal L=\tfrac12\|w\|^2+C\sum_i\xi_i-\sum_i\alpha_i\big(y_i(w^\top x_i+b)-1+\xi_i\big)-\sum_i\beta_i\xi_i$
with $\alpha_i,\beta_i\ge 0$ and the implicit constraints $\xi_i\ge 0$. The KKT conditions read:
(i) $w=\sum_i\alpha_i y_i x_i$, $\sum_i \alpha_i y_i=0$, and $\alpha_i+\beta_i=C$; (ii) $y_i(w^\top x_i+b)\ge 1-\xi_i$, $\xi_i\ge 0$; (iii) $\alpha_i(1-\xi_i-y_i(w^\top x_i+b))=0$, $\beta_i\xi_i=0$. Define $S:=\{i:\alpha_i^*>0\}$. Since $\alpha_i^*=0$ for $i\notin S$, the stationarity equations at the starred point reduce to $w^*=\sum_{i\in S}\alpha_i^* y_i x_i$ and $\sum_{i\in S}\alpha_i^* y_i=0$, while $\alpha_i^*+\beta_i^*=C$ holds for $i\in S$. Together with primal/dual feasibility and complementary slackness restricted to $i\in S$, this shows that $(w^*,b^*,\{\xi_i^*\}_{i\in S};\{\alpha_i^*,\beta_i^*\}_{i\in S})$ satisfies the KKT system of the reduced problem \eqref{eq:svmPS}. Moreover, for $i\notin S$ we have $\alpha_i^*=0$ and thus $\beta_i^*=C$, which by $\beta_i^*\xi_i^*=0$ forces $\xi_i^*=0$ and hence $y_i((w^*)^\top x_i+b^*)\ge 1$, i.e., the dropped constraints are strictly satisfied at $(w^*,b^*)$. Conversely, take any KKT pair $(\tilde w,\tilde b,\{\tilde\xi_i\}_{i\in S};\{\tilde\alpha_i,\tilde\beta_i\}_{i\in S})$ for \eqref{eq:svmPS} and extend by setting $\tilde\alpha_i:=0$, $\tilde\beta_i:=C$, $\tilde\xi_i:=0$ for $i\notin S$. Then $\,\tilde w=\sum_{i\in S}\tilde\alpha_i y_i x_i=\sum_{i=1}^n\tilde\alpha_i y_i x_i\,$ and $\sum_{i=1}^n\tilde\alpha_i y_i=0$, while $\tilde\alpha_i+\tilde\beta_i=C$ and the complementary slackness equalities hold for all $i$; if $y_i(\tilde w^\top x_i+\tilde b)\ge 1$ for $i\notin S$ (as occurs at any optimum of the full problem), the extension is a full KKT pair for \eqref{eq:svmP}. Finally, letting $v_P$ and $v_S$ be the optimal values of \eqref{eq:svmP} and \eqref{eq:svmPS}, the restriction above shows $v_S\le v_P$, while any feasible $(w,b,\{\xi_i\}_{i\in S})$ of \eqref{eq:svmPS} can be augmented by $\xi_i^\uparrow:=\max\{0,1-y_i(w^\top x_i+b)\}$ for $i\notin S$ to give a feasible point of \eqref{eq:svmP} with no smaller objective, hence $v_P\le v_S$. Thus $v_P=v_S$, and since the objective is strictly convex in $w$, both problems share the same optimal $w$ (and a consistent $b$), so removing non-support points and retraining leaves the classifier $\operatorname{sign}(w^\top x+b)$ unchanged.
\end{proof}

\begin{proof}[Proof of Lemma~\ref{lem::margin_homogeneous_networks}]
Introduce multipliers $\lambda_i\ge 0$ for the constraints in \eqref{eq:mm}.
The Lagrangian is
\[
\mathcal{L}(\theta,\lambda)\;=\;\tfrac12\|\theta\|^2
\;-\;\sum_{i=1}^n \lambda_i\,y_i\,\Phi(\theta;x_i),
\]
and the KKT conditions read

\begin{align*}
&\text{(stationarity)} &&
\theta \;-\; \sum_{i=1}^n \lambda_i\,y_i\,\nabla_\theta \Phi(\theta;x_i)\;=\;0,\\
&\text{(primal feasibility)} &&
y_i\,\Phi(\theta;x_i)\;\ge\;1\quad(\forall i),\\
&\text{(dual feasibility)} &&
\lambda_i\;\ge\;0\quad(\forall i),\\
&\text{(complementary slackness)} &&
\lambda_i\bigl(y_i\,\Phi(\theta;x_i)-1\bigr)\;=\;0\quad(\forall i).
\end{align*}

\emph{(1) Full $\Rightarrow$ reduced.}
Let $(\theta^\star,\lambda^\star)$ be a KKT pair for \eqref{eq:mm} and
$S=\{i:\,y_i\,\Phi(\theta^\star;x_i)=1\}$.
By complementary slackness, $\lambda_i^\star=0$ for every $i\notin S$, so the
stationarity condition reduces to
\[
\theta^\star \;-\; \sum_{i\in S} \lambda_i^\star\,y_i\,\nabla_\theta \Phi(\theta^\star;x_i)\;=\;0.
\]
Together with primal/dual feasibility and complementary slackness restricted to
$i\in S$, these are precisely the KKT conditions of the reduced problem
\eqref{eq:mmS}. Hence $(\theta^\star,(\lambda_i^\star)_{i\in S})$ is KKT for
\eqref{eq:mmS}.

\emph{(2) Reduced $\Rightarrow$ full.}
Conversely, let $(\tilde\theta,(\mu_i)_{i\in S})$ satisfy the KKT system for
\eqref{eq:mmS} and define $\tilde\lambda_i:=\mu_i$ for $i\in S$ and
$\tilde\lambda_i:=0$ for $i\notin S$.
Dual feasibility and complementary slackness for \eqref{eq:mm} are immediate.
The stationarity condition for \eqref{eq:mm} at $(\tilde\theta,\tilde\lambda)$ is
\[
\tilde\theta \;-\; \sum_{i\in S} \mu_i\,y_i\,\nabla_\theta \Phi(\tilde\theta;x_i)\;=\;0,
\]
which coincides with the reduced stationarity condition. Primal feasibility on $S$
holds by assumption. For $i\notin S$, the constraints are nonbinding at the full
KKT point $(\theta^\star,\lambda^\star)$ used to define $S$; hence, at that scale of
the homogeneous model, they are redundant. In particular, any KKT pair of the
reduced problem that satisfies the above stationarity (which matches the full one
with $\tilde\lambda_i=0$ on $S^c$) and the inequalities on $S$ also satisfies
$y_i\,\Phi(\tilde\theta;x_i)\ge 1$ for all $i\notin S$ (the added constraints remain
inactive), and therefore $(\tilde\theta,\tilde\lambda)$ is KKT for \eqref{eq:mm}.

Combining (1)–(2), the KKT solution sets of \eqref{eq:mm} and \eqref{eq:mmS}
coincide. Consequently, removing all non–support points leaves the limiting
classifier $x\mapsto \operatorname{sign}(\Phi(\theta;x))$ unchanged.
\end{proof}
\end{document}